\documentclass[10pt,letterpaper]{article}
\usepackage{times}
\usepackage{ijcai20}
\usepackage[hidelinks]{hyperref}
\usepackage[small]{caption}
\usepackage{xparse}
\usepackage{amsmath}
\usepackage{amsthm}
\usepackage{amssymb}
\usepackage{xspace}
\usepackage{relsize}
\usepackage{graphicx}
\usepackage{url}
\usepackage{math_commands}
\usepackage{adjustbox}

\newtheorem{theo}{Theorem}

\makeatletter
\let\@myref\ref

\newcommand{\refsec}[1]{Sec.\,\@myref{#1}}
\newcommand{\refseq}[1]{Sec.\,\@myref{#1}}
\newcommand{\refig}[1]{Figure\,\@myref{#1}}
\newcommand{\refigs}[2]{Figure\,\@myref{#1}-\@myref{#2}}

\newcommand{\reftbl}[1]{Table \@myref{#1}}
\newcommand{\refstep}[1]{Step \@myref{#1}}
\newcommand{\refalgo}[1]{Algorithm \@myref{#1}}
\newcommand{\refline}[1]{line \@myref{#1}}
\newcommand{\refchap}[1]{Chapter \@myref{#1}}
\newcommand{\reflst}[1]{List \@myref{#1}}
\newcommand{\refeq}[1]{Eq. \@myref{#1}}

\newcommand{\refex}[1]{Example \@myref{#1}}

\makeatother

\usepackage{stackengine}
\stackMath
\newcommand\tsup[2][2]{%
 \def\useanchorwidth{T}%
  \ifnum#1>1%
    \stackon[-.5pt]{\tsup[\numexpr#1-1\relax]{#2}}{\scriptscriptstyle\sim}%
  \else%
    \stackon[.5pt]{#2}{\scriptscriptstyle\sim}%
  \fi%
}

\newcommand{\brackets}[1]{{\left<#1\right>}}
\newcommand{\braces}[1]{{\left\{#1\right\}}}
\newcommand{\parens}[1]{{\left(#1\right)}}

\newcommand{\lsota}{state-of-the-art\xspace}  %lower-case version
\newcommand{\domind}{domain-independent\xspace}

\newcommand{\astar}{\xspace {$A^*$}\xspace}

\newcommand{\pddl}[1]{\textsf{\small #1}}

\def\_{\\[-0.3em]}

\makeatletter

\newcommand{\newheuristic}[2]{%
 \def#1{%
  \ifmmode%
  h^\text{#2}\xspace%
  \else%
  \text{#2}\xspace%
  \fi%
 }%
}

\newheuristic{\lmcut}{LMcut}
\newheuristic{\mands}{M\&S}
\newheuristic{\pdb}{PDB}
\newheuristic{\ff}{FF}
\newheuristic{\ce}{CEA}
\newheuristic{\cg}{CG}
\newheuristic{\ad}{add}
\newheuristic{\lc}{LC}
\newheuristic{\hmax}{max}

\newcommand{\newUnitCostHeuristic}[2]{%
 \def#1{%
  \ifmmode%
  \hat{h}^\text{#2}\xspace%
  \else%
  \text{#2}\xspace%
  \fi%
 }%
}

\newUnitCostHeuristic{\lmcuto}{LMcut}
\newUnitCostHeuristic{\mandso}{M\&S}
\newUnitCostHeuristic{\ffo}{FF}
\newUnitCostHeuristic{\ceo}{CEA}
\newUnitCostHeuristic{\cgo}{CG}
\newUnitCostHeuristic{\ado}{add}
\newUnitCostHeuristic{\gco}{gc}
\newUnitCostHeuristic{\lco}{LC}

\makeatother

\def\latentplanner{Latplan\xspace}

\renewcommand{\to}{\rightarrow}

\renewcommand{\iff}{\Leftrightarrow}

\newcommand{\function}[1]{\textsc{#1}}

\newcommand{\init}{{\vx^{I}}}
\newcommand{\goal}{{\vx^{G}}}
\newcommand{\zinit}{{\vz^{I}}}
\newcommand{\zgoal}{{\vz^{G}}}
\newcommand{\encode}{\function{Enc}}
\newcommand{\decode}{\function{Dec}}

\newcommand{\Tr}{X}

\newcommand{\defaultindex}{i}
\newcommand{\tr}[1][\defaultindex]{\mathsf{x}^{#1}}
\newcommand{\ztr}[1][\defaultindex]{\mathsf{z}^{#1}}
\newcommand{\xbefore}[1][\defaultindex,0]{\vx^{#1}}
\newcommand{\xafter}[1][\defaultindex,1]{\vx^{#1}}
\newcommand{\xbeforerec}[1][\defaultindex,0]{{\tsup[1]{\vx}}^{#1}}
\newcommand{\xafterrec}[1][\defaultindex,1]{{\tsup[1]{\vx}}^{#1}}

\newcommand{\xafterrecrec}[1][\defaultindex,1]{{\tsup[2]{\vx}}^{#1}}
\newcommand{\zbefore}[1][\defaultindex,0]{\vz^{#1}}
\newcommand{\zafter}[1][\defaultindex,1]{\vz^{#1}}
\newcommand{\zafterrec}[1][\defaultindex,1]{\tsup[1]{\vz}^{#1}}
\newcommand{\action}[1][\defaultindex]{\va^{#1}}

\newcommand{\aaee}{\function{action}}
\newcommand{\aaed}{\function{apply}}

\newcommand{\pre}[1]{\function{pre}\parens{#1}}
\newcommand{\adde}[1]{\function{add}\parens{#1}}
\newcommand{\dele}[1]{\function{del}\parens{#1}}
\newcommand{\effect}[1]{\function{effect}\parens{#1}}

\newcommand{\BN}[1]{\function{BN}\parens{#1}}

\newcommand{\GS}{\function{GS}}

\author{
Masataro Asai$^1$ \And
Christian Muise$^2$
\affiliations
$^1$ MIT-IBM Watson AI Lab, IBM Research, Cambridge USA\\
$^2$ School of Computing, Queen's University, Kingston Canada
}

\title{Learning Neural-Symbolic Descriptive Planning Models\\ via Cube-Space Priors: The Voyage Home (to STRIPS)}

\begin{document}
\maketitle

\begin{abstract}
We achieved a new milestone in the difficult task of enabling agents to learn
about their environment autonomously.
Our neuro-symbolic architecture is trained end-to-end to produce a
succinct and effective discrete state transition model from images
alone. Our target representation (the Planning Domain Definition
Language) is already in a form that off-the-shelf solvers can consume,
and opens the door to the rich array of modern heuristic search capabilities.
We demonstrate how the sophisticated innate prior we
place on the learning process significantly reduces the complexity of the learned
representation, and reveals a connection to the
graph-theoretic notion of ``cube-like graphs'', thus opening the door to a
deeper understanding of the ideal properties for learned
symbolic representations.
We show that the powerful domain-independent heuristics allow our system
to solve visual 15-Puzzle instances
which are beyond the reach of blind search,
without resorting to the Reinforcement Learning approach
that requires a huge amount of training on the domain-dependent reward information.
\end{abstract}

\textit{Extended version: This manuscript is slightly extended to include more discussion in the related works section.}

\section{Introduction}
\label{sec:introduction}

Learning a symbolic and descriptive transition model of an environment from unstructured and noisy input (e.g. images) is a major challenge in Neural-Symbolic integration. Doing so in an unsupervised manner requires solving both the Symbol Grounding \cite{taddeo2005solving} and the Action Model Learning/Acquisition problem, and is particularly difficult without reusing manually defined symbols.

Recent work that learns the discrete planning models from images has
opened a direction for applying the symbolic Automated Planning systems to the raw, noisy data \cite[Latplan]{Asai2018}.
The system builds on a bidirectional mapping between the visual perceptions and the propositional states; using separate networks for modeling action applicability and effects.
Latplan opened the door to applying a variety of interesting symbolic
methods to real world data.
E.g., its search space
was shown to be compatible with symbolic Goal Recognition 
\cite{amado2018goal}.

\begin{figure*}[tb]
 \center
 \includegraphics[width=0.9\linewidth]{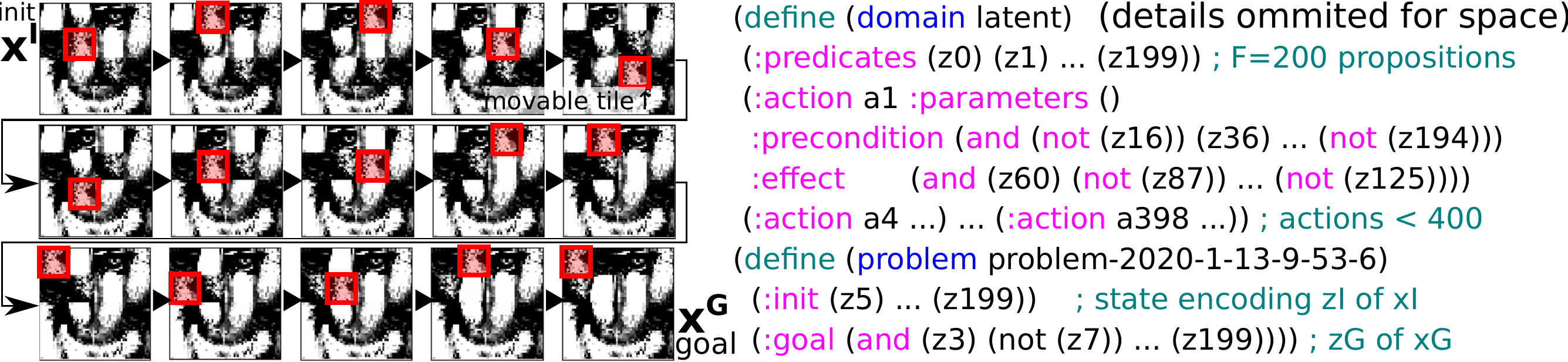}
 \caption{
(left) A 14-step optimal plan for a 15-puzzle instance generated by our system using off-the-shelf Fast Downward with $\lmcut$
using the PDDL generated by our system.
(right) The intermediate PDDL output from our NN-based system.
}
\label{15puzzle}
\end{figure*}

One major drawback of the previous work was that it used a non-descriptive, black-box neural model as the successor generator.
Not only that such a black-box model is incompatible with the existing heuristic search techniques,
but also, since a neural network is able to model a very complex function,
its direct translation into a compact logical formula
via a rule-based transfer learning method turned out futile \cite{Asai2019d}:
The model complexity causes an exponentially large grounded action model that cannot be processed by the modern classical planners.
Thus, \emph{obtaining the descriptive action models from the raw observations with minimal human interference} is
the next key milestone for expanding the scope of applying Automated Planning to the raw unstructured inputs.%,

We propose the \emph{Cube-Space AutoEncoder} (Cube-Space AE) neuro-symbolic architecture,
which addresses the complexity in the action model
by jointly learning a state representation and an action model of a \emph{restricted} class.
The purpose of the joint training is to 
learn an appropriate latent/state space where low-complexity action models exist,
exploiting the flexibility of the end-to-end training of neural networks (NNs).
The action models searched by the NN are restricted
to a graph class called \emph{directed cube-like graph} that
corresponds precisely to the STRIPS semantics.
Cube-Space AE allows us to directly extract the action effects for each action,
providing a grounded PDDL \cite{pddlbook} that is immediately usable by off-the-shelf planners.
We demonstrate its planning performance on visual puzzle domains including
15 Puzzle instances (\refig{15puzzle}).
Remarkably, the vast majority of states seen in the solution plans \textit{do not
appear in the training data}, i.e.,
the learned representation captures the underlying dynamics
and generalizes extremely well to the individual domains studied.

\section{Preliminaries}
\label{preliminary}

We denote a multi-dimensional array in bold and
its elements with a subscript (e.g., $\vx\in \R^{N\times M}$, $\vx_2 \in \R^M$),
an integer range $n<i<m$ by $n..m$,
and the $i$-th data point of a dataset by a superscript $^{i}$ which we may omit for clarity.
Functions (e.g. $\log,\exp$) are applied to the arrays element-wise.

Let $\mathcal{F}(V)$ be a propositional formula consisting of
 logical operations $\braces{\land,\lnot}$,
 constants $\braces{\top,\bot}$, and
 a set of propositional variables $V$.
We define a grounded (propositional) STRIPS Planning problem
as a 4-tuple $\brackets{P,A,I,G}$
where
 $P$ is a set of propositions,
 $A$ is a set of actions,
 $I\subseteq P$ is the initial state, and
 $G\subseteq P$ is a goal condition.
Each action $a\in A$ is a 3-tuple $\brackets{\pre{a},\adde{a},\dele{a}}$ where
 $\pre{a} \in \mathcal{F}(P)$ is a precondition and
 $\adde{a}$, $\dele{a}$ are the add-effects and delete-effects, where $\adde{a} \cap \dele{a} = \emptyset$.
A state $s\subseteq P$ is a set of true propositions,
an action $a$ is \emph{applicable} when $s$ \emph{satisfies} $\pre{a}$,
and applying an action $a$ to $s$ yields a new successor state
$a(s) = (s \setminus \dele{a}) \cup \adde{a}$.

\latentplanner is a framework for
\emph{domain-independent image-based classical planning} \cite{Asai2018}.
It learns the state representation and the transition rules
entirely from image-based observations of the environment with deep neural networks
and solves the problem using a classical planner.

\latentplanner is trained on a \emph{transition input} $\Tr$: a set of pairs of raw data randomly sampled from the environment.
The $\defaultindex$-th pair in the dataset $\tr=(\xbefore, \xafter) \in \Tr$ is
a randomly sampled transition from a environment observation $\xbefore$ to another observation $\xafter$ as the result of some unknown action.
Once trained, \latentplanner can process the \emph{planning input} $(\init, \goal)$, a pair of raw data images
 corresponding to an initial and goal state of the environment.
The output of \latentplanner is a data sequence representing the plan execution
 $(\init,\ldots \goal)$ that reaches $\goal$ from $\init$.
While the original paper used an image-based implementation, % (i.e., ``data'' = raw images),
conceptually any form of temporal data that can be auto-encoded to a learned representation is viable for this methodology.

\latentplanner works in 3 steps.
In Step 1, a \emph{State AutoEncoder} (SAE) (\refig{three}, left) neural network learns a bidirectional mapping between raw data $\vx$ (e.g., images)
 and propositional states $\vz\in\braces{0,1}^F$, i.e., the $F$-dimensional bit vectors.
The network consists of two functions $\encode$ and $\decode$, where 
$\encode$ encodes an image $\vx$ to $\vz=\encode(\vx)$, and $\decode$ decodes $\vz$ back to an image $\tsup[1]{\vx}=\decode(\vz)$.
The training is performed by minimizing the reconstruction loss $||\tsup[1]{\vx}-\vx||$ under some norm (e.g., Mean Square Error for images).

In order to guarantee that $\vz$ is a binary vector, the network must use a differentiable discrete activation function
such as Heaviside $\function{step}$ Function with straight-through estimator \cite{koul2018learning,bengio2013estimating} or Gumbel Softmax (GS) \cite{jang2016categorical},
which we use for its superior accuracy \cite[Table 3]{jang2016categorical}.
GS is an annealing-based continuous relaxation of $\argmax$ (returns a 1-hot vector), defined as
 $\GS(\vx)=\function{Softmax}((\vx-\log(-\log \vu)))/\tau)$,
and its special case limited to 2 categories \cite{MaddisonMT17} is
$\function{BinConcrete}(\vx)=\function{Sigmoid}((\vx+\log \vu-\log(1-\vu))/\tau)$,
where $\vu$ is a vector sampled from $\function{Uniform}(0,1)$ and $\tau$ is an annealing parameter.
As $\tau\rightarrow 0$, both functions approach to discrete functions:
$\GS(\vx)\to\argmax(\vx)$
and $\function{BinConcrete}(\vx)\to\function{step}(\vx)$. %  (step function thresholded at 0) -- assume known.

The mapping $\encode$ from $\braces{\ldots\xbefore, \xafter\ldots}$ to $\braces{\ldots\zbefore, \zafter\ldots}$
provides the propositional transitions $\ztr=(\zbefore,\zafter)$.
In Step 2, an Action Model Acquisition (AMA) method learns an action model from $\ztr$.
In Step 3, a planning problem instance $(\zinit, \zgoal)$ is generated from the planning input $(\init, \goal)$
and the classical planner finds the path connecting them.
In the final step, \latentplanner obtains a step-by-step, human-comprehensible visualization of the plan execution
by $\decode$'oding the intermediate states of the plan into images.
For evaluation, we use domain-specific validators for the visualized results
because the representation learned by unsupervised learning is not directly verifiable.

The original Latplan paper proposed two approaches for AMA.
AMA$_1$ is an oracular model that directly generates a PDDL without learning,
and AMA$_2$ is a neural model that approximates AMA$_1$ by learning from examples, which we mainly discuss.
AMA$_2$ consists of two neural networks: Action AutoEncoder (AAE) and Action Discriminator (AD).

AAE (\refig{three}, middle) is an autoencoder that learns to cluster the state transitions into a (preset) finite number of action labels.
Its encoder, $\aaee(\zbefore,\zafter)$, takes a propositional state pair $(\zbefore,\zafter)$ as the input
and returns an action.
The last layer of the encoder is activated by a categorical activation function (Gumbel Softmax) to become a one-hot vector of $A$ categories,
 $\action\in\braces{0,1}^A$ ($\sum_{j=1}^{j=A} \action_j = 1$), where 
$A$ is a hyperparameter for the maximum number of action labels and $\action$ represents an action label.
For clarity, we use the one-hot vector $\action$ and the index $a^\defaultindex=\arg \max \action$ interchangeably.
AAE's decoder takes the current state $\zbefore$ and $\action$ as the input and outputs $\zafterrec$,
acting as a progression function $\aaed(\action,\zbefore)=\zafterrec$.
AAE is trained to minimize the successor reconstruction loss $||\zafterrec-\zafter||$.

\begin{figure*}[tb]
 \centering
 \includegraphics[width=0.9\linewidth]{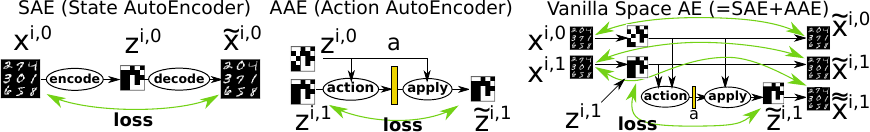}
 \caption{The illustration of State AutoEncoder, Action AutoEncoder, and the end-to-end combinations of the two.}
 \label{three}
\end{figure*}

AD is a PU-learning \cite{elkan2008learning} binary classifier that learns the preconditions of the actions
from the observed state transitions $P$ and the ``fake'' transitions $U$
sampled by applying random actions generated by the AAE.

Combining AAE and AD yields a successor function for graph search algorithms (e.g. \astar \cite{hart1968formal})
by enumerating the potential successor states $\zafter[a]=\function{apply}(\zbefore[],\action[])$ over $a\in 1..A$,
then prunes the states classified as invalid by the AD.
The major drawback of this approach is that both AAE and AD are black-box functions
incompatible with the standard PDDL-based planners and heuristics, and thus
requires custom blind search solvers.

\section{Cube-Space AutoEncoder}
\label{cube}

The issue in the mathematical model of AAE is that it allows for an arbitrarily complex state transition model.
The traditional STRIPS progression $\aaed(s,a)= (s \setminus \dele{a}) \cup \adde{a}$
\emph{disentangles} the effects from the current state $s$,
i.e., the effect $\adde{a}$ and $\dele{a}$ are defined entirely based on action $a$,
while AAE's black-box progression $\aaed(\action,\zbefore)$ does not offer such a separation,
allowing the model to learn arbitrarily complex conditional effects \cite{pddlbook}.
Due to this lack of separation,
the straightforward logical translation of the AAE with a rule-based learner
(e.g. Random Forest) requires a conditional effect for every effect bit,
resulting in a PDDL that cannot be processed by the modern classical planners due to the huge file size \cite{Asai2019d}.
In order to make the action effect independent from the current state,
we restrict the transition model learned by the network to a \emph{directed cube-like graph}.

A \emph{cube-like graph} $G(S,D)=(V,E)$ \cite{payan1992chromatic} is a simple undirected graph where
each node $v\in V$ is a finite set $v\subset S$,
$D$ is a family of subsets of $S$,
and for every edge $e = (v,w) \in E$, the symmetric difference
 $d = v\oplus w = (v\setminus w) \cup (w\setminus v)$ must belong to $D$.
For example, a unit cube is a cube-like graph because
$S=\braces{x,y,z},
V=\braces{\emptyset,\braces{x},\ldots \braces{x,y,z}},
E=\braces{(\emptyset,\braces{x}),\ldots (\braces{y,z},\braces{x,y,z})},
D=\braces{\braces{x},\braces{y},\braces{z}}$.
Since the set-based representation has a corresponding bit-vector, e.g.,
$V'=\braces{(0,0,0),(0,0,1),\ldots (1,1,1)}$,
we denote a one-to-one $F$-bit vector assignment as $f: V \to \braces{0,1}^F$.
Cube-like graphs have a key difference from normal graphs.

\begin{theo} %berge1991short,misra1992constructive
 \emph{\textbf{(1, \cite{vizing1965chromatic})}} Let the edge chromatic number $c(G)$ of an undirected graph $G$ % $G=(V,E), \forall v,w\in V; (v,w)\in E \iff (w,v)\in E$
 be the number of colors in a minimum edge coloring.
 Then $c(G)$ is either $\Delta$ or $\Delta+1$, where $\Delta$ is the maximum degree of the nodes.
 \emph{\textbf{(2)}} Mapping $E\to D$ provides an edge coloring and thus $c(G)\leq \min_f |D|$.
 \emph{\textbf{(3a)}} The minimum number of actions required to model $G$ is $c(G)$ with conditional effects,
 \emph{\textbf{(3b)}} and $\min_f 2|D|$ without.
\end{theo}

\begin{proof}
\textbf{(2)} $f$ is one-to-one: $w\not=w'\iff f(w) \not= f(w')$.
For any set $X$, $f(w) \not= f(w') \iff X\oplus f(w) \not= X\oplus f(w')$.
For any adjacent edges $(v,w)$ and $(v,w')$,
$w\not=w'$ because $G$ is simple (at most one edge between any nodes), thus
$f(v)\oplus f(w) \not= f(v)\oplus f(w')$.
The remainder follows from the definition.
\textbf{(3a)} % For each color $c\in 1..c(G)$ and
For each edge $(v,w)$ colored as $c$,
we add a conditional effect to $c$-th action using $f(v)$ and $f(w)$
as the conditions and the effects (see \refig{fig:star}, left).
\textbf{(3b)} Each $d\in D$ needs 2 actions for forward/backward directions.
\end{proof}

(1) and (3a,b) indicate that conditional effects can
compact as many edges as possible into just $A=\Delta$ or $\Delta+1$ actions
regardless of the nature of the transitions, while STRIPS effects cannot.
In (2), equality holds for hypercubes, and there are graph instances where $c(G) < \min_f 2|D|$ (\refig{fig:star}, right).
We ``proved'' this with an Answer Set Programming solver Potassco \cite{gebser2011potassco}.
(Detail omitted.)

\begin{figure}[tb]
 \centering
 \includegraphics[width=\linewidth]{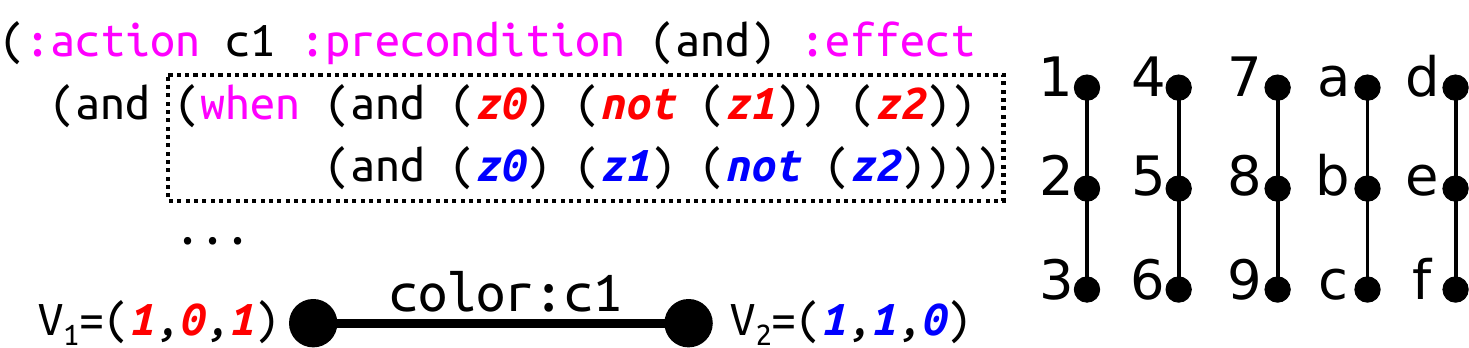}
 \caption{
(left): Single conditional effect can encode an arbitrary transition.
(right): This graph
 has $c(G)=2$.
 It has $F=4$ bit-vector unique assignments onto the nodes, but none satisfies $|D|=2$.
 An assignment with $|D|=4$ exists.
}
 \label{fig:star}
\end{figure}

Our NN architecture is based on a \emph{directed cube-like graph},
a directed extension of cube-like graph.
For every edge $e = (v,w) \in E$,
there is a pair of sets $d=(d^+,d^-)=(w \setminus v, v \setminus w)\in D$ which satisfies the asymmetric difference
$w=(v \setminus d^-) \cup d^+$
(similar to the planning representation, we assume $d^+ \cap d^- = \emptyset$).
It is immediately obvious that this graph class corresponds to the STRIPS action model without preconditions.
With preconditions, the state space is its subgraph.
This establishes an interesting theoretical connection between our innate
priors and the graph theoretic notion. We leave a formal analysis
of this connection to future work.

In order to constrain the learned latent state space of the environment,
we propose \emph{Cube-Space AutoEncoder}.
We first explain a vanilla \emph{Space AutoEncoder}, an architecture that
jointly learns the state and the action model
by combining the SAE and the AAE into a single network.
We then modify the $\aaed$ progression to form a cube-like state space.

The vanilla Space AutoEncoder (\refig{three}, right) connects the SAE and AAE subnetworks.
The necessary modification, apart from connecting them, is the change in loss function.
In addition to the loss for the successor prediction in the latent space,
we also ensure that the predicted successor state can be decoded back to the correct image $\xafter$.
Thus, the total loss is a sum of:
(1) the main reconstruction losses $||\xbefore-\xbeforerec||_2^2$ and $||\xafter-\xafterrec||_2^2$,
(2) the successor state reconstruction loss (\emph{direct loss}) $||\zafter-\zafterrec||_1$,
(3) the successor image reconstruction loss $||\xafter-\xafterrecrec||_2^2$,
and (4) the regularization loss.

An important technique for successfully training the system is to employ a bootstrap phase commonly used in the literature:
We delay the application of the direct loss until a certain epoch,
in order to address the relatively short network distance between the two latent spaces $\zafter/\zafterrec$ compared to $\xafter/\xafterrecrec$.
If we enable the direct loss from the beginning, the total loss does not converge because
$\zafterrec$ prematurely converges to $\zafter$ causing a mode-collapse (e.g. all 0),
before the image en/decoder learns a meaningful latent representation.

\begin{figure}[tb]
 \includegraphics[width=\linewidth]{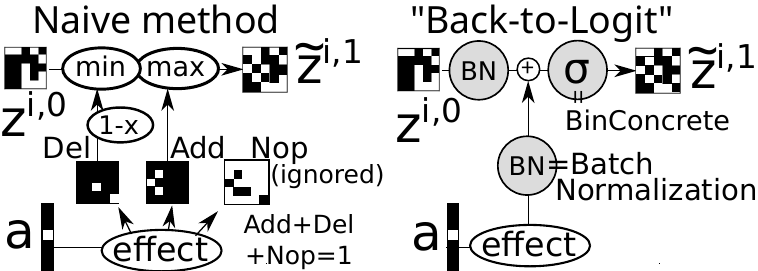}
 \caption{A naive and a Back-to-Logit implementation of the $\aaed$ module of the Cube-Space AE.}
 \label{fig:cube}
\end{figure}

Cube-Space AE modifies the $\aaed$ network so that it directly predicts the effects
 without taking the current state as the input, and logically computes the successor state
based on the predicted effect and the current state.
For example,
a naive implementation of such a network is shown in \refig{fig:cube} (left).
The \function{effect} network predicts a binary tensor of $F\times 3$,
which is activated by $F$-way Gumbel-Softmax of 3 categories.
Each Gumbel Softmax corresponds to one bit in the $F$-bit latent space
and 3 classes correspond to the add effect, delete effect and NOP,
only one of which is selected by the one-hot vector.
The effects are applied to the current state
either by a max/min operation or its smooth variants, $\smax(x,y)=\log (e^x+e^y)$ and $\smin(x,y)=-\smax(-x,-y)$.

While being intuitive, we found these naive implementations extremely difficult to train.
Our final contribution to the architecture is \emph{Back-to-Logit} (BtL, \refig{fig:cube}, right),
\emph{a generic approach that computes the logical operation in the continuous logit space}.
We re-encode a logical, binary vector back to a continuous representation
by an \emph{order-preserving}, monotonic function $m$, so that it preserves its original meaning.
We then perform the logical operation by the arithmetic addition to the continuous vector produced by the $\function{effect}$ network.
Finally we re-activate the result with a discrete activation (e.g. \GS/\function{BinConcrete}).
Formally, 
$\zafterrec=\aaed(\zbefore,\action)=\function{BinConcrete}(\effect{\action}+m(\zbefore))$.
$\function{BinConcrete}$ becomes $\function{step}$ after the training.
Notice that it guarantees the STRIPS-style effects
for the set of transitions $\{\ldots (\zbefore,\zafter) \ldots\}$ generated by the same action $\action[]$:
\begin{theo}
$\function{add}$: $\exists i;(\zbefore_f,\zafter_f)=(0,1)\Rightarrow\forall i;\zafter_f=1$.
$\function{del}$ (proof omitted):
$\exists i;(\zbefore_f,\zafter_f)=(1,0)\Rightarrow\forall i;\zafter_f=0$.
\end{theo}
\begin{proof}
$\zafter_f=\function{step}(m(\zbefore_f)+e)$ where $e=\effect{\action[]}_f$.
From the assumption,
$1=\function{step}(m(0)+e)$, therefore $m(0)+e>0$.
Then $m(1)+e>0$ holds from monotonicity.
Therefore $\function{step}(m(\zbefore_f)+e) = 1$ regardless of $\zbefore_f$.
\end{proof}

This theorem allows us to extract the effects of each action $a$ 
directly from the encoded states by
$\adde{a}=\{f|\exists i;\ztr_f=(0,1)\}$ and
$\dele{a}=\{f|\exists i;\ztr_f=(1,0)\}$.

We found that an easy and successful way to implement $m$
is \emph{Batch Normalization} \cite{ioffe2015batch}, a method that was originally developed for
addressing the covariate shift in the deep neural networks.
During the batch training of the neural network,
Batch Normalization layer $\BN{\vx}$ takes a batched input of $B$ vectors $\braces{\vx^{i}\ldots \vx^{i+B}}$,
computes and maintains the element-wise mean and variance of the input $\vx$ across the batch dimension
(e.g. $\frac{1}{B}\sum_{k=i}^{k=i+B}\vx^{k}_j$ for the $j$-th element of $\vx$),
then shift and scale $\vx$ element-wise so that the result has a mean of 0 and a variance of 1.
We apply this to both the effect and the state vectors, i.e.,
$\aaed(\zbefore,\action)=\function{BinConcrete}(\BN{\effect{\action}}+\BN{\zbefore})$.

\paragraph{Precondition Learning.}

Thanks to the strong structural prior,
precondition learning with a baseline static bits extraction method
turned out to be enough to plan effectively:
$\function{pre}(a)=\{f|\forall i;\zbefore_f=1\}\cup \{\lnot f|\forall i; \zbefore_f=0\}$
for the $\zbefore$ satisfing $\aaee(\zbefore,\zafter)=a$.
This simple procedure, which learns a single decision-rule from the dataset,
achieves the sufficient success rate in the empirical evaluation.
Improving the accuracy of the precondition learning is a matter of ongoing investigation.

\section{Evaluation}

\begin{table*}[tb]
\centering
\begin{adjustbox}{width={\linewidth},keepaspectratio}
\begin{tabular}{|l|}
\\
Domain   \\
Hanoi    \\ 
LOut     \\
Twisted  \\ 
Mandrill \\ 
Mnist    \\ 
Spider   \\ 
\end{tabular}
\begin{tabular}{c|c|c|}
\multicolumn{3}{c|}{\emph{(1)} BtL Cube-Space AE} \\
Rec. & Succ. & Direct  \\
.001 & .002  & .001    \\ 
.000 & .000  & .000    \\ 
.000 & .001  & .000    \\ 
.000 & .001  & .001    \\ 
.000 & .000  & .000    \\ 
.001 & .001  & .001    \\ 
\end{tabular}
\begin{tabular}{r|r|r|}
\multicolumn{3}{c|}{\emph{(2)} Total loss}\\
MinMax  & Smooth &  \textbf{BtL} \\
.439    & .436   & \textbf{.003} \\ 
.506    & .506   & \textbf{.000} \\ 
.458    & .487   & \textbf{.001} \\ 
.500    & .600   & \textbf{.002} \\ 
.512    & .506   & \textbf{.001} \\ 
.607    & .563   & \textbf{.003} \\ 
\end{tabular}
\begin{tabular}{r|r|r|}
\multicolumn{3}{c|}{\emph{(3)} Ablation study }\\
 -$\function{BN}$ & -Direct & -Succ.  \\
 .019             & .088    & .501    \\ 
 .002             & .239    & .412    \\ 
 .002             & .154    & .498    \\ 
 .035             & .046    & .495    \\ 
 .002             & .158    & .462    \\ 
 .026             & .073    & .376    \\ 
\end{tabular}
\begin{tabular}{c|c|}
\multicolumn{2}{c|}{\emph{(4)} Direct loss under $A=300\to 100$} \\
\multicolumn{1}{c|}{Cube-Space AE} & \multicolumn{1}{c|}{SAE+Cube-AAE} \\
.001 $\to$ .001 & .001 $\to$ .002  \\ 
.001 $\to$ .008 & .010 $\to$ .040  \\ 
.000 $\to$ .002 & .001 $\to$ .011  \\ 
.001 $\to$ .002 & .000 $\to$ .006  \\ 
.000 $\to$ .001 & .002 $\to$ .005  \\ 
.003 $\to$ .002 & .001 $\to$ .003  \\ 
\end{tabular}
\end{adjustbox}
 \caption{
\emph{(1)}
The reconstruction loss for each output of the Cube-Space AE on the test set.
``Rec.'', ``Succ.'', ``Direct'' stands for
$||\xbefore-\xbeforerec||_2^2$, $||\xafter-\xafterrecrec||_2^2$, $||\zafter-\zafterrec||_1$ respectively.
We do not show $||\xafter-\xafterrec||_2^2$ because the results are similar to ``Rec.''.
``Succ.'' tends to be less accurate than ``Rec.'' because it is affected by the successor state prediction.
\emph{(2)}
Back-to-Logit (BtL) Cube-Space AE outperforms naive methods (\refig{fig:cube}, left) on the total loss.
\emph{(3)}
Ablation study of BtL Cube-Space AE on the total loss, showing that
Batch Normalization (-BN), Direct loss (-Direct), and the successor image loss (-Succ.) are all essential.
\emph{(4)}
Comparing the effect of the number of actions $A$ on the direct loss with the Cube-Space AE and the SAE + Cube-AAE.
The latter learns the state and the action representation separately with the same $\aaed$ as the Cube-Space AE.
When $A=100$, Cube-AAE tends to perform significantly worse than Cube-Space AE.
}
\label{bigtable}
\end{table*}

\begin{figure*}[tb]
 \centering
\begin{adjustbox}{width={\linewidth},keepaspectratio}
\begin{minipage}[b]{0.30\linewidth}
 \centering
 4x4 LightsOut and 8 puzzle \\
 \includegraphics[width=0.48\linewidth]{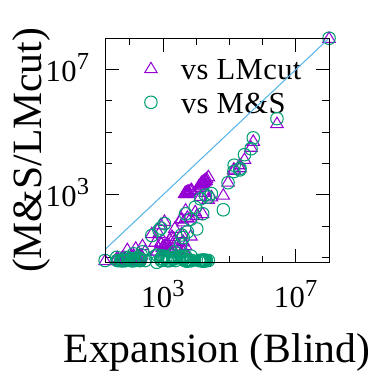}
 \includegraphics[width=0.48\linewidth]{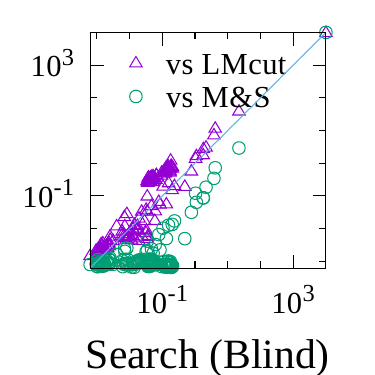}
\end{minipage}
\begin{minipage}[b]{0.30\linewidth}
 \centering
 Mandrill 15-puzzle \\
 \includegraphics[width=0.48\linewidth]{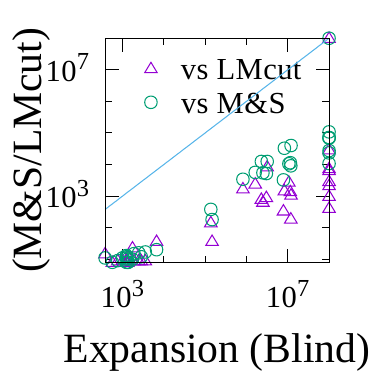}
 \includegraphics[width=0.48\linewidth]{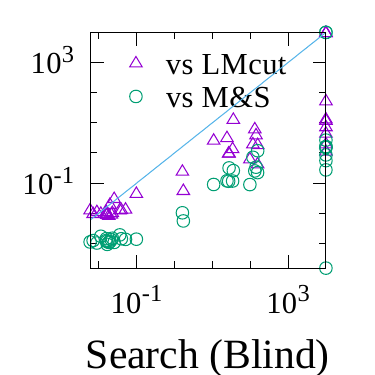}
\end{minipage}
 \includegraphics[width=0.17\linewidth]{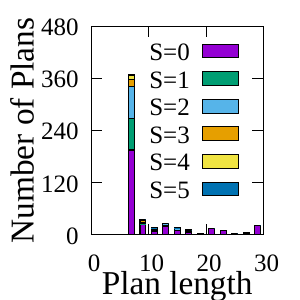}
 \includegraphics[width=0.17\linewidth]{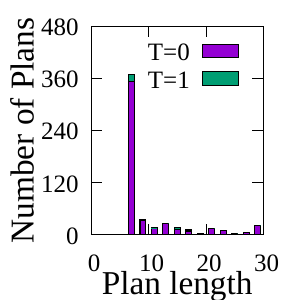}
\end{adjustbox}
 \caption{
\emph{(Columns 1-2)}
Comparison of the number of node expansions and the search time [sec] between blind, $\mands$ and $\lmcut$.
Both \lsota admissible heuristics achieve the significant reduction in the node expansion.
The search time was not reduced with $\lmcut$ due to its high evaluation cost.
$\mands$ also incurs the high initialization cost.
\emph{(3-4)}
However, for 15 puzzles, this contributes to more instances solved in the resource limit (unsolved instances on the borders).
\emph{(5-6)}
Solutions that contain $S$ states / $T$ transitions in the training data. (Plans with length $>30$ are included in column 30.)
}
\label{runtime-comparison}
\label{ood-comparison}
\end{figure*}

\begin{table*}[tb]
\centering
\setlength{\tabcolsep}{0.4em}
\begin{adjustbox}{width={\linewidth},keepaspectratio}
\begin{tabular}{|rrr|rrr|rrr|rrr|}
\multicolumn{3}{|c|}{blind} & \multicolumn{3}{c|}{\gco}&\multicolumn{3}{c|}{blind} & \multicolumn{3}{c|}{\gco} \\
\textbf{f} & \textbf{v} & \textbf{o} &\textbf{f} & \textbf{v} & \textbf{o} &\textbf{f} & \textbf{v} & \textbf{o} &\textbf{f} & \textbf{v} & \textbf{o} \\
0  & 0  & 0  & 8  & 4  & 3           & 1  & 0  & 0  & 19 & 0  & 0  \\ 
0  & 0  & 0  & 4  & 0  & 0           & 0  & 0  & 0  & 14 & 2  & 2  \\ 
15 & 14 & 14 & 17 & 17 & \textbf{16} & 15 & 15 & 15 & 13 & 12 & 12 \\ 
0  & 0  & 0  & 0  & 0  & 0           & 0  & 0  & 0  & 0  & 0  & 0  \\ 
8  & 8  & 8  & 4  & 4  & 4           & 3  & 3  & 1  & 13 & 3  & 0  \\ \hline
\multicolumn{6}{|c|}{SAE+AAE+AD} & \multicolumn{6}{c|}{SAE+CubeAAE+AD} 
\end{tabular}
\begin{tabular}{c|}
\\
Domain   \\
LOut(30)     \\
Twisted(30)  \\ 
Mandrill(30) \\ 
Mnist(30)    \\ 
Spider(30)   \\ 
\hline
15puzzle(40) $\to$ \\
\end{tabular}
\begin{tabular}{rrr|rrr|rrr|rrr|rrr|}
\multicolumn{3}{c|}{blind} & \multicolumn{3}{c|}{\gco} & \multicolumn{3}{c|}{lama} & \multicolumn{3}{c|}{\lmcut} & \multicolumn{3}{c|}{\mands} \\
\textbf{f} & \textbf{v} & \textbf{o} &\textbf{f} & \textbf{v} & \textbf{o} &\textbf{f} & \textbf{v} & \textbf{o} &\textbf{f} & \textbf{v} & \textbf{o} &\textbf{f} & \textbf{v} & \textbf{o} \\
30 & 30 & 30 & 30 & 30 & 17 & 30 & 21 & 5  & 30 & 30          & 30          & 30 & 30          & 30          \\ 
30 & 25 & 25 & 30 & 26 & 5  & 30 & 20 & 4  & 30 & \textbf{29} & \textbf{29} & 30 & 25          & 25          \\ 
30 & 29 & 13 & 30 & 18 & 9  & 30 & 23 & 11 & 30 & 29          & 13          & 30 & \textbf{30} & 13          \\ 
30 & 25 & 24 & 30 & 4  & 4  & 30 & 7  & 7  & 30 & 25          & 24          & 30 & \textbf{26} & 24          \\ 
30 & 28 & 16 & 30 & 27 & 12 & 28 & 22 & 14 & 30 & 28          & 16          & 30 & 28          & 16          \\ \hline
32 & 32 & 29 & 28 & 9  & 9  & 38 & 13 & 3  & 38 & \textbf{38} & \textbf{33} & 38 & \textbf{38} & \textbf{33} \\ 
\end{tabular}
\end{adjustbox}
 \caption{
Plans f(ound), v(alid) and o(ptimal) %, out of 30 instances.
by (left) AMA$_2$ and (right) Cube-Space AE.
Best result except $>2$ ties are in bold.
}
\label{planning}
\end{table*}

We evaluated our approach on the dataset used by \citeauthor{Asai2018}, which consists of 6 image-based domains.
\textbf{MNIST 8-Puzzle}
is a 42x42 pixel image-based version of the 8-Puzzle, where tiles contain hand-written digits (0-9) from the  MNIST database \cite{lecun1998gradient}.
Valid moves in this domain swap the ``0'' tile  with a neighboring tile, i.e., the ``0'' serves as the ``blank'' tile in the classic 8-Puzzle.
The \textbf{Scrambled Photograph 8-Puzzle (Mandrill, Spider)} cuts and scrambles real photographs, similar to the puzzles sold in stores.
\textbf{LightsOut} is           %\cite{anderson1998turning}
a 36x36 pixel game where a 4x4 grid of lights is in some on/off configuration,
and pressing a light toggles its state as well as the states of its neighbors.
The goal is all lights Off.
\textbf{Twisted LightsOut} distorts the original LightsOut game image by a swirl effect
that increases the visual complexity.
\textbf{Hanoi} is a visual Tower of Hanoi with 9 disks and 3 towers in 108x32 pixels.
In all domains, we sampled 5000 transitions and
divided them into 90\%,5\%,5\% for the training, validation/tuning, and testing, respectively.
(Code: guicho271828/latplan@Github.)

We evaluate the proposed approach via an ablation study.
For each of the 6 datasets, we trained the proposed Cube-Space AE and its variants for 200 epochs, batch size 500,
GS temperature $\tau: 5.0 \to 0.7$ with exponential schedule,
with Rectified Adam optimizer \cite{liu2019variance}.
We evaluate the network with \emph{total loss},
which is the sum of the Mean Square Error (MSE) losses for the image outputs
($||\xbefore-\xbeforerec||_2^2+||\xafter-\xafterrec||_2^2+||\xafter-\xafterrecrec||_2^2$)
and the Mean Absolute Error (MAE) direct loss for successor state prediction ($||\zafter-\zafterrec||_1$).
During evaluation, Gumbel Softmax is treated as $\argmax$ without noise, following the previous work \cite{Asai2019a}.

We tuned the hyperparameters for the validation dataset using the sequential genetic algorithm in the public Latplan codebase
(population 20, mutation rate 0.3, uniform crossover and point mutation).
It selects from
the learning rate $r\in\braces{0.1,0.01,0.001}$,
the latent space size $F\in\braces{100,200,1000}$,
the number of actions $A\in\braces{100,200,400,800,1600}$,
layer width $W\in\braces{100,300,600}$ and depth $D\in\braces{1,2,3}$ of \function{action/apply},
the scaling parameter for
the Zero-Suppress loss  $\alpha\in\braces{0.1,0.2,0.5}$ \cite{Asai2019a},
the variational loss    $\beta\in\braces{-0.3,-0.1,0.0,0.1,0.3}$ \cite{higgins2017beta},
and the direct MAE loss $\gamma\in\braces{0.1,1,10}$,
and the bootstrap epoch for $\alpha,\gamma$: $d_{\alpha},d_{\gamma}\in\braces{10,20,40,60,100}$
(scaling parameters are set to 0 before it).
For each domain, we searched for 100 iterations ($\approx$15min/iter, 24 hours total) on a Tesla K80.

We first tested several alternative architectural considerations,
especially the effect of Back-to-Logit (BtL)
that combines the logit-level addition and batch normalization for implementing a logical operation in the latent space.
We considered 3 architectures that predict the successor state $\zafterrec$.
In \reftbl{bigtable}\emph{(2)},
``MinMax''/``Smooth'' are the variants discussed in (\refig{fig:cube}, left):
$\max(\min (\zbefore, 1-\function{del}(\va)), \function{add}(\va))$ and its smooth variants.
``BtL'' (\refig{fig:cube}, right: proposed approach) is
$\function{BinConcrete}(\function{BN}(\zbefore)+\function{BN}(\function{effect}(\va)))$.
BtL convincingly outperformed the other options.

Next we performed an ablation study. We tested
``-BN'', which does not use Batch Normalization,  % $\function{BinConcrete}(\zbefore+\function{effect}(\va))$,
 ``-direct'', which relies solely on the image-based loss for predicting the successor,
and ``-Succ'', which similarly relies solely on the state-based direct loss.
\reftbl{bigtable}\emph{(3)} shows that
they fail to achieve the small total loss, indicating all components are essential.

Lastly, we tested if the end-to-end training helps obtaining the compact action representation.
We trained a ``Cube-AAE'', an AAE that has the same structure as the Cube-Space AE's $\aaed$,
but is trained on the fixed state dataset obtained by a SAE separately.
Cube-AAE must learn the STRIPS effects within the fixed state representation.
We tuned the hyperparameter (100 iterations) except $A=100,300$ (fixed).
As $A$ gets smaller, Cube-AAE performs significantly worse than Cube-Space AE (\reftbl{bigtable}\emph{(4)}), as expected.
This is because Cube-AAE cannot model the transitions with fewer actions
by reshaping the state representation to have more recurring effects,
which characterize the cube-like graph.

\subsection{Evaluation in the Latent Space}

We ran the off-the-shelf planner Fast Downward on the PDDL files generated by our system.
All domains have the fixed goal state $\goal$.
Initial states $\init$ are sampled from the frontier of a Dijkstra search from the goal at $g$-value $=l$,
the shortest path length.
This makes the task harder than the $l$-step random walks used in previous work which allow shorter paths.
In each domain, we generated 30 instances for $l=7$.
We tested
\astar with blind heuristics,
Goal count (\gco)~\cite{fikes1972strips},
\lmcut \cite{Helmert2009},
Merge-and-Shrink (M\&S) \cite{HelmertHHN14}, and
the first iteration of the satisficing planner LAMA \cite{richter2010lama}.
In order to make sure that we have the best representation,
we added 200 more genetic algorithm iterations.
We compared our approach with AMA$_2$ system (SAE+AAE+AD \cite{Asai2018}) with blind and \gco
(the only heuristics available in AMA$_2$).
Experiments are run % on Intel Xeon E5-2600 cluster
with the 15 minutes time limit and 2GB memory limit.

\reftbl{planning} shows that our system outperforms AMA$_2$ overall.
The solutions \textbf{f}(ound) are \textbf{v}(alid) more often,
and \textbf{o}(ptimal) with admissible heuristics.
Note that, even with the admissible functions,
the optimality is guaranteed only with respect to
the potentially imperfect search graph generated by the Space AEs.
We also observed shorter runtime / fewer expansions
with the sophisticated heuristic functions like \lmcut/\mands (\refig{runtime-comparison}, \emph{column 1-2}),
which strongly indicates that the models
we are inferring contain similar properties as the human-designed PDDL models.
Also,
we additionally verified the generality of the representation learned by the system.
We compared the resulting plans generated by our system against the training dataset and
counted how often the states/transitions could be just ``memorized'' from training data.
\refig{ood-comparison} \emph{(col. 5-6)} shows that such states are rare, % across the experiment,
indicating that the planner is inducing new, unseen states/transitions.
In Hanoi, the plans tend to be invalid or the goal is unreachable, even for easy instances
($l=3$, $\textbf{v}=3$, $\textbf{f}=32$, 100 instances, blind).
Hanoi requires accurate preconditions because 
any pair of states have only a single shortest path.
Addressing its accuracy is left for future work.

Finally,
we tested Mandrill 15-puzzle, a significantly more challenging 4x4 variant of the sliding tile puzzle (\refig{15puzzle}).
We trained the network with more hyperparameter tuning iterations (300) and a larger training set (50000).
We generated $l=14,21$ instances (20 each) and ran the system (\reftbl{planning}, bottom right).
Blind search hits the memory limit on 6 instances which \mands, \lmcut solved successfully.
\refig{runtime-comparison} (\emph{col. 3-4}) confirms the overall reduction in the node expansion and the search time.
\textbf{To our knowledge, this is the first demonstration of the domain-independent heuristics in an automatically learned symbolic latent space.}

\section{Related Work}

\subsection{Traditional action learners}

Traditional action learners require the partially hand-crafted symbolic input
\cite{YangWJ07,CresswellMW13,aineto2018learning,zhuo2019learning,CresswellG11,MouraoZPS12,zhuo2013action,KonidarisKL14,KonidarisKL15,andersen2017active}
although the requirements of the systems may vary.

For example,
some systems require state sequences \cite{YangWJ07}
which contain a set of predicate symbols (such as \pddl{on}, \pddl{clear} for \pddl{blocksworld}),
while others require the sample action sequences \cite{CresswellG11,CresswellMW13,KonidarisKL14,KonidarisKL15,andersen2017active}
which contain a set of action symbols (such as \pddl{board}, \pddl{depart} in \pddl{miconic} domain,
or \pddl{go-left/right}, \pddl{up-ladder} in Treasure Game domain).
They may also require the object and type symbols, such as \pddl{block-A/B} and \pddl{block} in \pddl{blocksworld} domain.
Some supports the noisy input \cite{MouraoZPS12,zhuo2013action},
partial observations in a state
and/or missing state/actions in a plan trace \cite{aineto2018learning}, or
a disordered plan trace \cite{zhuo2019learning}.
Recently, \citeauthor{bonet2020learning} (\citeyear{bonet2020learning})
proposed a method that takes a state space topology -- that is, a non-factored state space graph
where each node is opaque and is represented by an ID, and each edge is assigned an action label/symbol
which originates from human labelling (e.g. \pddl{move}/\pddl{drop} in \pddl{gripper} domain).
Also, the state IDs represent the ``same-ness''/identities of the states already.
It is non-trivial to determine such an identifier in the noisy, high-dimensional continuous state space such as images.
Interestingly, however, the approach is able to generate predicate symbols and object symbols (and thus propositional symbols)
without human-provided labels.

Approach-wise, they can be grouped into 4 categories:
MAX-SAT based approach \cite{YangWJ07,zhuo2013action,zhuo2019learning},
Object-centric approach \cite{CresswellG11,CresswellMW13},
Learning-as-planning approach \cite{aineto2018learning},
and abstract-subgoal option approach \cite{KonidarisKL14,KonidarisKL15,andersen2017active}.

\textbf{MAX-SAT based approaches} try to find the most reasonable
planning operator definitions that matches the symbolic input observations,
by encoding the hypothesis space and the input dataset as a MAX-SAT problem \cite{vazirani2013approximation}.

\textbf{Object-centric approaches} builds a Finite State Machine
(FSM) of each object type from the action and objects in the input sequence,
assuming that lifting the objects into typed parameters will result in the same FSM.

\textbf{Learning-as-Planning} approach is similar to MAX-SAT approach,
but models the problem as classical planning problem of sequentially constructing
the action model which maximally meets the observations.

\textbf{Abstract Subgoal Option} approaches \cite{KonidarisKL14,KonidarisKL15,konidaris2018skills,andersen2017active}
developped a series of techniques for generating planning models,
from a set of state-option sequences $\Pi=\braces{(s_0,o_0,s_1,o_1\ldots)}$
obtained from the environment and the intermediate states collected during the continuous execution of options.
Each input state $s\in S$ is a vector value of a mixture of continuous and discrete variables $V$.
Due to the authors' main background (and audience) in MDP,
this line of work uses \emph{options} (macro action, or high-level action) instead of actions,
which allows the collection of state data during the execution of each option.
These intermediate states are then used as the negative examples in the training of the framework.

Since the input sequence contains human-assigned option symbols $o\in O$ (e.g. $o_0=\pddl{up-ladder}$),
it relies on symbolic input and thus
they do not address the action symbol grounding addressed in this paper and the previous Latplan.
Also, in typical robotic applications, MDPs as well as the experiments conducted in their work,
options (defined by initiation condition, termination condition, and a local policy) are hand-coded,
e.g., each condition over the low-level state vectors is written by a robot expert,
and are used for collecting data.
Moreover, Latplan framework does not rely on the negative examples.
As discussed in the original Latplan paper \cite{Asai2018},
ideally negative examples may include physically infeasible situtations which precludes the data collection,
thus explicitly relying on negative examples has a limitation.

For each option label, the framework tries to learn the effects and the preconditions.
Effects are represented as an \emph{abstract subgoal option},
which is a (\emph{mask}, \emph{subgoal option}) pair.
\emph{subgoal option} unconditionally sets values to all state variables,
while a \emph{mask} limits the variables to be affected.

There are several variations of the approach, but overall they work as follows, in the following order:
\textbf{(1)}
First, convert the set of sequences
$\Pi=\braces{(s_0,o_0,s_1,o_1\ldots)}$ into individual transitions
$T=\braces{(s_t,o_t,s_{t+1})}$.
\textbf{(2)}
Segment the transition dataset by each action label:
$T(o)=\braces{(s_t,o,s_{t+1})}$ for each label $o$
(e.g. $o=\pddl{up-ladder}$).  In the paper, further, the initiation set
$I(o)=\braces{s_t}$ and the effect set $E(o)=\braces{s_{t+1}}$
are obtained.
\textbf{(3)}
For each $o$, identify the list of state variables
$\textit{modify}(o)\subseteq V$ that are modified by $o$, by comparing
$s_t$ and $s_{t+1}$ in $T(o)$.
In some variations \cite{KonidarisKL15,andersen2017active},
this part could be learned by a probabilistic and/or bayesian model.
\textbf{(4)}
Segment the state variables $V$ into a family of mutually-independent
subsets $F=\braces{f\subseteq V}$ that covers $V$ ($V=\bigcup_F f$,
$i\not=j\Rightarrow f_i\cap f_j=\emptyset$), \emph{minimizing} the
number of such subsets ($|F|$), \emph{subject to} each subset $f$
uniquely maps to a set of actions that modifies it, i.e., for a mapping
$\textit{modifiedBy}: F\rightarrow O$, \textbf{(4.a)} $\forall o \in
\textit{modifiedBy}(f); f\subseteq \textit{modify}(o)$, and
\textbf{(4.b)} $i\not=j \Rightarrow \textit{modifiedBy}(f_i) \not=
\textit{modifiedBy}(f_j)$.  The subsets are called ``factors'', and
could also be called subspaces.
Each factor represents a set of variables which always change values at once.
\textbf{(5)}
Extract the state vectors in the entire dataset $\Pi$ as $\cal S$.
\textbf{(6)}
For each $f$, project $\cal S$ onto $f$, i.e., extract the axes of the state vectors
corresponding to $f$. Let's call the result ${\cal S}_f$.
\textbf{(7)}
Cluster each ${\cal S}_f$ into a finite set of $k$ clusters.
The number of clusters depends on the actual clustering method.
In a discrete environment, it is possible to enumerate all combination of values, where $k$ is the number of combinations.
The resulting clusters could be used either as a multi-valued representation in SAS+, or
as a set of propositional variables, which the author calls symbols.
Actions modify the state vector and therefore the clusters whose projections belong to.
Therefore, the cluster membership
behaves in a manner similar to a Finite State Machine or a Domain Transition Graph of each SAS+ variable.
\textbf{(8)}
Compute the effects from the effect set $E(o)$, its projection to each $f$ and its cluster membership.
Due to the abstract subgoal option assumption,
changes in the cluster membership can be directly encoded as the effect.
It could also use the intermediate states as the negative examples.
\textbf{(9)}
Compute the preconditions by
extracting the cluster membership shared by all states in the initiation set $I(o)$.
This is performed for each factor/subspace $f$, i.e., on the state dataset projected on $f$ and its clusters,
and the preconditions are added only if the cluster membership remains static in the projected states.
In some variations \cite{KonidarisKL15}, this step can be combined with a classifier that returns a probability,
trained with positive and negative examples.
In some variations \cite{KonidarisKL15,andersen2017active}, this step may include further clustering of options,
and the precondition of a single option may have a disjunction of multiple conditions.
\textbf{(10)}
Finally, the resulting representation is converted into PDDL.

Different versions of the approach rely on a variety of traditional machine learning model such as
C4.5 Decision Tree (in \cite{KonidarisKL14}),
Support Vector Machine and
Kernel density estimation (in \cite{KonidarisKL15}),
Bayesian sparse Dirichlet-categorical model and
Bayesian Hierarchical Clustering (in \cite{andersen2017active}),
or sparse Principal Component Analysis in the robotic task using point clouds (in \cite{konidaris2018skills}).
Therefore, they either work directly on the observed states, or linear latent variables.
The factoring process requires that the input state vector consists of disentangled/orthogonal subspaces,
which does not hold for noisy high-dimensional dataset like images, where most or all of input dimensions are entangled.

\subsection{Action Model Acquisition from Natural Language corpus}

There are several lines of work that extracts a PDDL action model from a
natural language corpus. Framer \cite{lindsay2017framer} uses a CoreNLP
language model while EASDRL \cite{feng2018extracting} uses Deep
Reinforcement Learning \cite{dqn}. The difference from our approach is
that they are reusing the symbols found in the corpus, i.e., \emph{parasitic} \cite{harnad1990symbol},
while we generate the discrete propositional symbols from the scratch, using only the visual perception 
which completely lacks such a predefined set of discrete symbols.

While some approaches extract an action model from a
natural language corpus \cite{lindsay2017framer,feng2018extracting},
they reuse the symbols in the corpus (referred to as \emph{parasitic} \cite{taddeo2005solving}).
Our approach requires only the visual perception % which completely lacks such a predefined set of symbols
and can generate the propositional/action symbols completely from scratch.

\subsection{Novelty-based planning without action description}

While there are recent efforts in handling the complex state space without
having the action description % and relying only on the state description
\cite{frances2017purely}, action models could be used for other purposes
such as Goal Recognition \cite{ramirez2009plan},
macro-actions \cite{BoteaB2015,ChrpaVM15},
or plan optimization \cite{chrpa2015exploiting}.

\subsection{Discrete modeling of the noisy input}

There are several lines of work that
learn the discrete representation of the raw environment.
Latplan SAE \cite{Asai2018} uses the Gumbel-Softmax VAE \cite{MaddisonMT17,jang2016categorical}
which was modified from the original to maximize the KL divergence term for the Bernoulli distribution \cite{Asai2019a}.
Causal InfoGAN \cite{kurutach2018learning} uses GAN\cite{goodfellow2014generative}-based approach
combined with Gumbel Softmax prior and Mutual Information prior.
Quantized Bottleneck Network \cite{koul2018learning} uses quantized activations (i.e., step functions)
with Straight-Through gradient estimator \cite{bengio2013estimating},
which enables the backpropagation through the step function.
There are more complex variations such as VQVAE \cite{van2017neural}, DVAE++\cite{vahdat2018dvae++}, DVAE\# \cite{vahdat2018dvae}.

\subsection{Other Related Work}

While Deep Reinforcement Learning has seen recent success in settings like
Atari \cite{dqn} and Go \cite{alphago},
it rely on the predefined action symbols in the simulator (e.g., levers, Fire button, grids).
Our work compliments it by learning a simulator of the environment.
While being similar to model-based RL, we generate both the transition function \emph{and the action space}.

Our BtL may resemble the residual technique repopularized by Resnet \cite{he2016deep},
which dates back to Kalman filters $s_{t+1}=X(s_t)+B(u_t)+w$
where $s,u,w$ are the state, controller and noise, and $X$ is sometimes an identity.
The key difference is that we operate in a propositional/discrete space.

Our framework works on a factored but non-structured propositional representation.
While we do not address the problem of lifting the action description,
combining these approaches with the FOL symbols (relations/predicates) found by NN \cite{Asai2019b}
using frameworks capable of learning a lifted condition, such as Neural Logic Machine \cite{dong2018neural},
is an interesting avenue for future work.

\section{Discussion and Conclusion}

In this work, we achieved a major milestone in
propositional Symbol Grounding and Action Model Acquisition:
A complete embodiment 
that converts the raw visual inputs
into a succinct symbolic transition model
in an automatic and unsupervised manner,
enabling the full potential of the modern heuristic search planners.
We accomplished this through the \emph{Cube-Space AutoEncoder} neural architecture
that jointly learns the discrete state representation and the descriptive action representation.
Our main contributions are the \emph{cube-like graph prior}
and its \emph{Back-to-Logit implementation}
that shape the state space to have a compact STRIPS model.

Furthermore, we demonstrated the first evidence 
of the scalability boost from the \lsota search heuristics
applied to the automatically learned state spaces.
These \domind functions, which have been a central focus of the planning community in the last two decades,
provide admissible guidance without learning. This is in contrast to
\emph{popular} reinforcement learning approaches that suffer from poor sample efficiency,
domain-dependence, and the complete lack of formal guarantees on admissibility.

Latplan requires uniform sampling from the environment,
which is nontrivial in many scenarios.
Automatic data collection via exploration is a major component of future work.

\end{document}